%% file: ijcai26.tex
\documentclass{article}
\input{packs}
\input{define}

\usepackage{ijcai26}

\usepackage{times}
\usepackage{soul}
\usepackage{url}
\usepackage[hidelinks]{hyperref}
\usepackage[utf8]{inputenc}
\usepackage{caption}
\usepackage{graphicx}
\usepackage{amsmath}
\usepackage{amsthm}
\usepackage{booktabs}
\usepackage{algorithm}
\usepackage{algorithmic}
\usepackage[switch]{lineno}

\newtheorem{theorem}{Theorem}

\title{Modeling Dynamic Mixtures of Time-Delay Systems from Streaming Time Series}

\author{
    Author Name
    \affiliations
    Affiliation
    \emails
    email@example.com
}

\author{
Ren Fujiwara
\and
Yasuko Matsubara\and
Yasushi Sakurai\\
\affiliations
SANKEN, The University of Osaka, Japan\\
\emails
\{r-fujiwr,yasuko,yasushi\}@sanken.osaka-u.ac.jp
}

\begin{document}

\maketitle
\begin{abstract}
\input{000abstract}
\end{abstract}

\section{Introduction}
\label{sec: 01_intro}
\input{010intro}

\section{Proposed Method}
\label{sec: 02_model}
\input{020model}

\section{Theoretical Analysis}
\label{sec:theory}
\input{030analysis}

\section{Experimental Results}
\label{sec: 04_exp}
\input{040experiments}

\section{Related Work}
\label{sec: 08_related}
\input{080related}

\section{Conclusion}
\label{sec: 09_conclusion}
\input{090conclusion}
\section*{Acknowledgements}
This work was supported by
JSPS KAKENHI Grant-in-Aid for Scientific Research Number
JP26H02499,
JST CREST JPMJCR23M3, 
JST  K Program JPMJKP25Y6,
JST COI-NEXT JPMJPF2009, 
JST COI-NEXT JPMJPF2115,
Future Social Value Co-Creation Project - Osaka University 
.
\bibliographystyle{named}
\bibliography{
BIB/DelayMix,
BIB/temp
}
\clearpage
\appendix
\section*{Appendix}
\input{0A_appendix}
\end{document}

%% file: packs.tex
\def\equationautorefname~#1\null{Equation~(#1)\null}
\usepackage{times}
\usepackage{soul}
\usepackage{url}
\usepackage[hidelinks]{hyperref}
\usepackage[utf8]{inputenc}
\usepackage[small]{caption}
\usepackage{graphicx}
\usepackage{booktabs}
\usepackage[switch]{lineno}

\usepackage{breakurl}
\usepackage{bm}
\usepackage{xspace}
\usepackage{algorithmic}
\usepackage{algorithm}
\usepackage{comment}
\usepackage{autobreak}

\usepackage{utfsym}

\usepackage{adjustbox}
\usepackage{array}
\usepackage{booktabs}
\usepackage{multirow}

\usepackage{array,graphicx}
\usepackage{booktabs}
\usepackage{pifont}

\definecolor{lightgray}{gray}{0.85}
\usepackage{multirow} 
\usepackage{mathtools}

\usepackage{breqn}
\usepackage{enumerate}
\usepackage{physics}

\usepackage{amsmath}
\usepackage{amsthm}
\usepackage{amsfonts}


%% file: define.tex
\newtheorem{problem}{Problem}
\newtheorem{definition}{Definition}

\newtheorem{lemma}{\textsc{Lemma}}

\newcommand{\mypara}[1]{\noindent{\textbf{#1.}}}

\newcommand{\method}{\text{DelayMix}\xspace}

\newcommand{\Mcollect}{DynamicMomentCollection\xspace}
\newcommand{\Madapt}{ModelDatabaseAdaption\xspace}
\newcommand{\Fpred}{FuturePrediction\xspace}

\newcommand{\datadim}{d\xspace}
\newcommand{\sdim}{k\xspace}

\newcommand{\cdim}{d_c\xspace}

\newcommand{\len}{l}

\newcommand{\data}{\mathbf{X}}
\newcommand{\ctr}{\mathbf{U}}

\newcommand{\datapoint}{x}
\newcommand{\ctrpoint}{u}
\newcommand{\datavec}{\mathbf{\datapoint}}
\newcommand{\ctrvec}{\mathbf{\ctrpoint}}
\newcommand{\statepoint}{z}
\newcommand{\statevec}{\mathbf{\statepoint}}

\newcommand{\estpoint}{\hat{x}}
\newcommand{\estvec}{\mathbf{\estpoint}}

\newcommand{\lstep}{l_s}
\newcommand{\fdata}{\data^f}

\newcommand{\lMat}{\mathbf{A}}

\newcommand{\cMat}{\mathbf{B}}

\newcommand{\oMat}{\mathbf{C}}

\newcommand{\lag}{\tau}

\newcommand{\sCov}{\mathbf{\Gamma}}
\newcommand{\oCov}{\mathbf{R}}

\newcommand{\newt}{t_n}
\newcommand{\trank}{R}

\newcommand{\msCov}{\mathbf{P}}
\newcommand{\msB}{\mathbf{V}}
\newcommand{\fwmu}{\hat{\boldsymbol \mu}}
\newcommand{\filmu}{\boldsymbol \mu}

\newcommand{\ord}{\mathcal{O}}

\newcommand{\iparams}{\{\lMat_i,\cMat_i, \oMat_i\}}
\newcommand{\regimeset}{\theta}
\newcommand{\datac}{\data^c}
\newcommand{\cctr}{\ctr^c}
\newcommand{\nregimes}{R\xspace}

\newcommand{\systemtensor}{\mathcal{S}}

\newcommand{\fone}{\mathbf{q}^{(1)}}
\newcommand{\ftwo}{\mathbf{q}^{(2)}}
\newcommand{\fthree}{\mathbf{q}^{(3)}}

\newcommand{\mone}{\mathbf{m}^{(1)}}
\newcommand{\mtwo}{\mathbf{m}^{(2)}}
\newcommand{\mthree}{\mathbf{m}^{(3)}}

\DeclareMathOperator*{\argmin}{\text{arg}\,\text{min}}

\definecolor{lightgray}{gray}{0.85}

%% file: 000abstract.tex
This research addresses the problem of adaptive modeling in time-series data streams with clear input-output relationships. This problem is challenging because rapid system changes (regime shifts) caused by environmental factors or input delay changes degrade model performance, and the trade-off among accuracy, robustness, and memory usage arises when using multiple small models for each time-series pattern. To address these issues, this paper presents an online framework/method that treats streaming time series as dynamic mixtures of time-delay systems. This framework maintains robustness of model tracking and reduces memory usage by summarizing past regimes using a fixed-length representation that captures both the system dynamics and input-output delays. Concretely, this approach constructs a summary system tensor using the system's Markov parameter series, capturing both dynamic behavior and delay characteristics. If necessary, a tensor decomposition algorithm extracts relevant past models from the tensor and helps select the system that best fits the current regime. This method enables rapid adaptation to environmental changes and is computationally efficient. Tests on real datasets show that \method consistently outperforms other methods, achieving superior forecast accuracy and faster adaptation to delays, especially for highly non-stationary data.

%% file: 010intro.tex
In recent years, the importance of modeling time-series data streams with explicit input-output relationships and modeling based on such data streams has grown. In particular, in real-world systems such as ship control~\cite{li2022hybrid}, industrial robots~\cite{lu2021gan}, data center cooling~\cite{datacentor}, diesel engine control~\cite{engine}, and autonomous driving~\cite{drivectr}, modeling failures can easily manifest as failures in safe and efficient operation, necessitating the development of methods for stable and efficient modeling on time-series data streams.

However, modeling on streams simultaneously faces the challenges of data nonstationarity and computational resource constraints, making it difficult to directly apply models and algorithms designed for offline use. Specifically, the following two problems arise: (1) Changes in the underlying system often cause abrupt model changes, or regime shifts, rapidly degrading the performance of trained models \cite{RegimeShift,RegimeShift2,RegimeShiftNature}. Furthermore, this challenge is further complicated by the unknown input delay inherent in systems with input-output relationships ~\cite{ARIMAX,MPC,kLinReg,QLag,SINDyc,DMDc,ReLiNet}. This means that changes in the model's input delays must be simultaneously tracked. (2) Time-series data streams have no temporal boundaries (i.e., they are semi-infinitely long), making model estimation using all past data prone to design errors. Furthermore, models designed using multiple small models for each time-series pattern suffer from a trade-off between robustness of accuracy and memory usage. While retaining more models stabilizes accuracy, having too many models increases the model search cost and memory computational complexity without bound.

For example, current stream prediction models that consider regime shifts offer scalable algorithms in terms of model adaptability and computational time on the stream \cite{RegimeCast,regime_other,ModePlait,MicroAdapt}. However, these methods monotonically increase the number of models, thereby increasing not only memory usage but also the computational time required for model switching. Furthermore, these methods do not consider exogenous variables. 
In other words, these methods lack design principles for tracking changes in input delay, making them difficult to directly apply to our problem.
Recent neural approaches, such as ReLiNet~\cite{ReLiNet} and TimeXer~\cite{TimeXer}, incorporate exogenous signals through switched linear systems and Transformer-based attention, respectively, achieving competitive results. However, these models generally do not model time delays explicitly and tend to struggle in streaming settings due to high computational cost and frequent retraining. Therefore, this paper aims to develop a method that simultaneously satisfies both adaptability to sudden environmental changes and scalable estimation. The research question in this paper is: 

\textit{Can we develop a method that achieves high-accuracy modeling of time-series data streams with explicit input-output relationships, while tracking rapid environmental changes and maintaining high computational efficiency?}

We propose \method
, an online framework that resolves the conflict between structural identification and streaming adaptation. This framework maintains high model tracking performance while avoiding increased memory usage by retaining model sets used in past regimes as mixtures of fixed-length summarized representations of the system's dynamic behavior and input-output delay structure. Specifically, we focus on the fact that the system's Markov parameter series integrates its dynamic behavior and input-output delay structure, and we construct a system tensor that summarizes them. Then, if necessary, we use a tensor decomposition algorithm to extract useful past models from the system tensor and adaptively select the system that best fits the current regime. This approach achieves both tracking rapid environmental changes and high computational efficiency.

In summary, \method is an online algorithm for modeling streaming time series as mixtures of time-delay systems. Our main contributions are as follows:
\begin{itemize}
\item \textbf{Streaming Modeling of Time-Delay Systems:} We propose a new formulation that leverages Markov parameter sequences to jointly learn unknown delays and system dynamics in streaming environments.
\item \textbf{Scalability with Resource Constraints:} The computation time of \method is independent of the time series length, thanks to incremental model updating, and thus is a faster algorithm than its competitors.
\item \textbf{Robustness Against Structural Drift:} Experiments on real-world datasets demonstrate that our method significantly outperforms state-of-the-art baselines in predictive accuracy, especially in environments with strong nonstationarity.
\end{itemize}

%% file: 020model.tex
In this paper, we study time series arising from mixtures of time-delay MIMO systems, where the delays are unknown and may differ across systems.
Our goal is to recover the dynamics of each system from one observed sequence and use them for online multi-step forecasting.
Rather than estimating delays directly, we estimate Markov parameters for each regime to summarize input–output behavior. For each regime, we then create a delay-free input–output equivalent state-space model, which we use for inference and prediction.

This section introduces the proposed model and the algorithm. First, in section~\ref{subsec:mixture-model}, we define our model. Then we formalize the problem definition in Problem~\ref{prob:online-forecasting}. We then present the proposed online algorithm \method, which includes sub-algorithms named \Mcollect, \Madapt, and \Fpred.
The theoretical justification for our method is given in Section~\ref{sec:theory}.

\subsection{Mixtures of Time-Delay Systems}
\label{subsec:mixture-model}
In practice, time-series data rarely come from just one stationary dynamical system.
Instead, the underlying dynamics often switch between several regimes, depending on hidden conditions
such as the operating mode, environment, or user behavior.
To capture this non-stationarity, we model the data as coming from a mixture of time-delay systems.

Let $\nregimes$ denote the number of regimes, and let $r_t \in \{1,\dots,\nregimes\}$ denote the (unobserved) regime index at time $t$.
For each regime $i \in \{1,\dots,\nregimes\}$, we introduce a latent state
$\statevec_{i}(t) \in \mathbb{R}^{\sdim}$ and system matrices
$
    \lMat_{i} \in \mathbb{R}^{\sdim \times \sdim},
    \cMat_{i} \in \mathbb{R}^{\sdim \times \cdim},
    \oMat_{i} \in \mathbb{R}^{\datadim \times \sdim}.
$
We also include an integer-valued delay $\lag_{i} \ge 0$ in the state update.
Conditioned on being in regime $r_t=i$ at time $t$, the dynamics evolve as
\begin{align}
    \statevec_{i}(t+1) & = \lMat_{i} \statevec_{i}(t) + \cMat_{i} \ctrvec(t-\lag_{i}),
    \label{eq:mixture_state}\\
    \datavec(t) & = \oMat_{i} \statevec_{i}(t),
    \label{eq:mixture_output}
\end{align}
where $\ctrvec(t) \in \mathbb{R}^{\cdim}$ is the observed exogenous input and
$\datavec(t) \in \mathbb{R}^{\datadim}$ is the observed output.
The input $\ctrvec(t)$ is observed at every time step, but the regime index $r_t$ and the latent
states $\statevec_{i}(t)$ are not directly observed. We use $\estvec(t)$ for model predictions
(for example, predictions produced by filtering) when needed later.

\mypara{Regime-specific Markov parameters and mixture view}
For each regime, the time-delay system \eqref{eq:mixture_state}--\eqref{eq:mixture_output} yields the following equivalent Markov parameters (impulse response):
\begin{align}
              \nonumber
    \left\{ \mathbf{g}^{(i)}_1, \mathbf{g}^{(i)}_2, \dots \right\}, \qquad
    \mathbf{g}^{(i)}_\ell = \oMat_{i} \bigl(\lMat_{i}\bigr)^{\ell-1} \cMat_{i}.
\end{align}
When an explicit delay $\lag_i$ exists in the system, the Markov parameters exhibit a characteristic pattern: in the ideal noise-free case, the first $\lag_i$ Markov parameters vanish, and nonzero responses begin at delay $\lag_i+1$.
This relationship between the Markov parameters and the time delay suggests that simply estimating the Markov parameters, without explicitly accounting for the delay, yields an estimate of the time-delay system. Furthermore, estimating the regimes represented by the time-delay system can be reformulated as the problem of estimating mixtures of Markov parameters from the data. Here, in \cite{LDSmeetsCTR}, it was theoretically shown that mixtures of Markov parameters in observed data can be extracted by decomposing a properly generated third-order tensor (i.e., the system tensor) as a superposition of rank-1 tensors. In this work, we construct an algorithm that applies this theoretical framework to regime estimation for time-delay systems in real-world streaming settings.
In the rest of this section, we focus on the proposed algorithm.

\begin{figure*}[t]
    \centering
    \includegraphics[width=\linewidth]{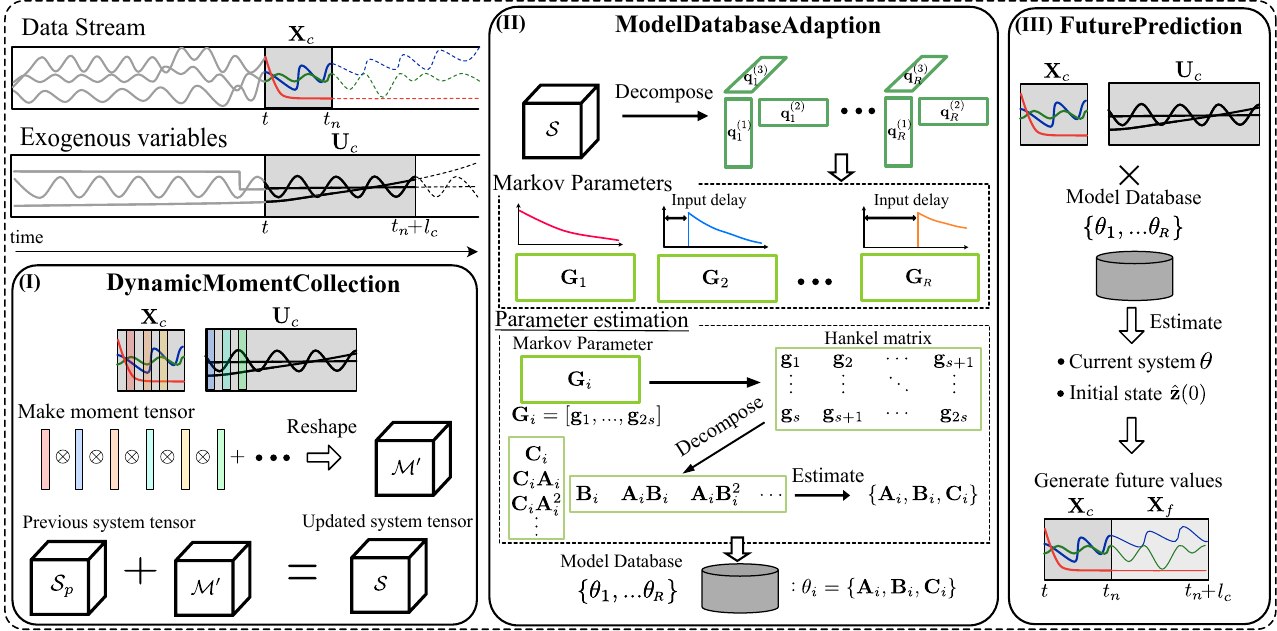}
    \caption{
        Overview of \method.
        Given the current data window $\datac$ and the current exogenous variables $\cctr$, the proposed method incrementally constructs a moment-based system tensor, decomposes it into regime-specific Markov parameters, realizes for each regime an equivalent delay-free state-space model, and uses these models to forecast future trajectories.
    }
    \label{fig:overview}
\end{figure*}

\subsection{Online Forecasting with Moment-Based System Tensors}
\label{subsec:online-forecasting}

We now explain how \method\ uses moment-based system tensors for online modeling and
the main idea is to keep an up-to-date empirical estimate of a system tensor at each update that summarizes the higher-order moments of $(\datavec,\ctrvec)$, and to periodically factorize this tensor to update a mixture of delay-free state-space models.
We then use these models to infer the current dynamical regime and to forecast future outputs
based on the exogenous inputs.

\subsubsection{Data Windows and Exogenous Variables}

We assume that the data stream $\data$, denoted as $\{\datavec(t)\}_{t \ge 1}$, arrives sequentially. 
The algorithm runs at set update times, processing a sliding window of the most recent data.

\begin{definition}[Current data window: $\datac$]
Let $\datac = \data[t:\newt]$ denote the subsequence of length $\len_c$ extracted
from the observed time series, starting at time index $t$ and ending at $\newt$.
\end{definition}

\begin{definition}[Current exogenous variables: $\cctr$]
Let $\cctr = \ctr[t:\newt+\lstep]$ denote the subsequence of exogenous inputs
available at the current update, where $\lstep$ is the forecasting horizon.
The segment $\ctr[t:\newt]$ is used for updating the model, and
$\ctr[\newt+1:\newt+\lstep]$ is used for forecasting.
\end{definition}

We focus on modeling streaming data as a mixture of time-delay systems and on using dynamics for forecasting, which is one of the primary applications of \method.
At each update, the algorithm observes the current data window $\datac$ and the corresponding
exogenous input sequence $\cctr$, along with the model estimated up to the previous update.
With this information, our goal is to refine the current system mixture and predict the $\lstep$-step-ahead outputs.

We formalize the problem as follows.

\begin{problem}
\label{prob:online-forecasting}
    \textbf{Given} the current data window $\datac$, the exogenous input sequence
    $\cctr$ sampled at regular time intervals, the previously estimated model
    parameters $\regimeset_p$, and the previous system tensor $\systemtensor_p$, \textbf{at each update} do the following:
    \begin{itemize}
        \item Update the system tensor $\systemtensor$ and the current
              system parameters
              \begin{align}
              \nonumber
                  \regimeset = \bigl\{\lMat_{c}, \cMat_{c}, \oMat_{c}\bigr\}
              \end{align}
              to better fit the newly observed data;
        \item Predict the future outputs $\fdata = \data[\newt+1:\newt+\lstep]$
              given the future inputs $\ctr[\newt+1:\newt+\lstep]$;
        \item Maintain computational efficiency and bounded memory usage as new data arrives.
    \end{itemize}
\end{problem}

To solve Problem~\ref{prob:online-forecasting}, \method\ organizes the online
estimation and forecasting pipeline into three parts:
\Mcollect\ (system tensor collection), \Madapt\ (online model adaptation via tensor
factorization and system realization), and \Fpred\ (forecasting trajectories with the
delay-free models).
We describe each part in turn. Figure~\ref{fig:overview} shows an overview of \method, from the current data window and exogenous inputs to the updated models and forecasts.

The goal of \Mcollect is to maintain an up-to-date empirical estimate of the global system tensor $\systemtensor$ using streaming input-output data. We build structured higher-order moments to capture the temporal dependencies from the underlying switching dynamical systems. Our approach uses a compact method inspired by~\cite{LDSmeetsCTR} that does not require storing raw data. \Mcollect keeps running estimates of the necessary moment blocks and updates $\systemtensor$ incrementally as new samples arrive. This keeps memory usage low while preserving a tensor whose CP decomposition can recover the regime-specific Markov parameters, up to scaling and permutation.
Concretely, for a sub-window start time $\tau$ and a triplet $(k_1,k_2,k_3)\in\{1,\dots,k_{\max}\}^3$, define

$
t_1=\tau+k_1,
t_2=\tau+k_1+k_2+1,
t_3=\tau+k_1+k_2+k_3+2,
\tilde t_1=\tau,
\tilde t_2=\tau+k_1+1,
\tilde t_3=\tau+k_1+k_2+2.
$
Using these indices, we form empirical sixth-order moments over three output--input pairs and, for computational efficiency, we group each pair into a single mode using
$\mathrm{vec}(\datavec(t_j)\ctrvec(t_j)^\top)(j=1,2,3)$, yielding a third-order moment block
$\mathcal{M}'(k_1,k_2,k_3)\in\mathbb{R}^{p\times p\times p}$.
We combine these blocks across valid $\tau$ and triplets to form the system tensor $\systemtensor$ using an incremental update.

\subsubsection{\Madapt: Online Model Adaptation via Tensor Factorization}

The goal of ModelDatabaseAdaption is to find and update the
set of dynamical systems \(\theta\) from a moment-based tensor representation.
While \Mcollect\ keeps updating $\systemtensor$, extracting separate system modes requires
a more involved tensor decomposition and system realization step. This step is triggered only when
a change in dynamics is detected or when enough new data have been collected.

The adaptation procedure has the following steps:
\begin{enumerate}
    \item \textbf{Tensor decomposition:}
          We apply the Alternating Least Squares (ALS) algorithm to the system tensor
          $\systemtensor$ to approximate it by a rank-$\trank$ CP decomposition
          \begin{align}
              \nonumber
              \systemtensor \approx \sum_{i=1}^{\trank}
                  \fone_i \otimes \ftwo_i \otimes \fthree_i,
          \end{align}
          where $(\fone_i, \ftwo_i, \fthree_i)$ are factor vectors associated with component $i$.
          Under the assumptions stated in Theorem~\ref{thm:moment-cp-compact}, the rank-one components
          can be associated with individual regimes and their stacked Markov parameters.
    \item \textbf{Markov parameter reconstruction and delay-free realization:}
          For each component $i$, we rearrange the factor vectors
          $(\fone_i, \ftwo_i, \fthree_i)$ into a sequence of estimated Markov parameters
          $\{\hat{\mathbf{g}}^{(i)}_k\}_{k=1}^{K}$ by following the construction of the system
          tensor (see also Section~\ref{sec:theory}).
          The initial zero pattern in this sequence implicitly encodes the effective delay for
          regime $i$, but we do not explicitly estimate the delay.
          Instead, we treat $\{\hat{\mathbf{g}}^{(i)}_k\}$ as the Markov parameters of an
          input--output equivalent system and apply a standard realization procedure, such as the
          Ho--Kalman algorithm \cite{HoKalman} to obtain a delay-free state-space model:
          \begin{align}
          \nonumber
              \bigl\{\lMat_{i}, \cMat_{i}, \oMat_{i}\bigl\}
              = \text{Ho-Kalman}\bigl(\{\hat{\mathbf{g}}^{(i)}_k\}_{k=1}^{K}\bigr).
          \end{align}
          By Theorem~\ref{thm:equivalence} in Section~\ref{sec:theory}, the resulting delay-free model
          $\bigl\{\lMat_{i}, \cMat_{i}, \oMat_{i}\bigl\}$ is input--output equivalent to a time-delay system and thus it faithfully captures the behavior of regime $i$.
          Finally, these parameters can be optionally fine-tuned so as to better fit the current
          data $\datac$, and the collection of all such systems constitutes the updated set
          $\regimeset$.
\end{enumerate}

In summary, \Madapt uses tensor factorization to separate regime-specific Markov parameters
from the system tensor and then realizes, for each regime, a delay-free state-space model that is suitable for inference and forecasting.

\subsubsection{\Fpred: Forecasting Trajectories with Delay-Free Models}

The goal of \Fpred is to infer a suitable initial state for the currently active system and use it, together with the learned models, to generate future values.
Given the updated set of delay-free systems $\regimeset$ and the current data window $(\datac,\cctr)$, we first infer the active regime and its latent states, for example, by running a bank of Kalman filters and smoothers, one per regime.
We utilize the forward-pass (filter) and backward-pass (smoother) equations.

The backward pass gives a smoothed estimate of the state trajectory, including an estimate
of the current (or initial) state $\hat{\statevec}(0)$ for the active delay-free system.
Using this state estimate and the identified system matrices $\bigl\{\lMat_{i}, \cMat_{i}, \oMat_{i}\bigl\}$,
we then simulate the delay-free dynamics forward in time, using the future inputs
$\ctr[\newt+1:\newt+\lstep]$, to generate the $\lstep$-step-ahead future values $\fdata$.

By Theorem~\ref{thm:equivalence},
each delay-free model
$\bigl\{\lMat_{i}, \cMat_{i}, \oMat_{i}\bigl\}$ obtained in this way is input--output equivalent to a
time-delay system representing regime $i$. This means forecasting with these delay-free models is
theoretically justified.

%% file: 030analysis.tex
In this section, we explain the theoretical foundations behind the modeling approach.
We also discuss the assumptions and algorithmic design choices introduced earlier.
We focus on two main components that are central to \method.
First, we represent each time-delay MIMO system, which corresponds to one regime in the mixture model,
using a delay-free state-space model built from its Markov parameters.
Second, we recover the regime-specific Markov parameters from a moment-based system
tensor using CP decomposition.
We then analyze the computational complexity of the online algorithm \method.

\subsection{Equivalence of a Time-Delay System to a Standard MIMO Model}
We begin by formalizing the relationship between a time-delay MIMO system and an
equivalent delay-free state-space model.
Recall from Section~\ref{subsec:mixture-model} that each regime $i$ in the mixture model
is described by a time-delay MIMO system. For clarity, we first focus on a single system and
drop the regime index $i$.
The equivalence result in this subsection supports the realization step in
\Madapt\ and the forecasting step \Fpred\ in \method.

Consider a single segment of a time-delay system described by
\begin{align}
    \statevec(t+1) & = \lMat'\statevec(t) + \cMat'\ctrvec(t-\lag), \label{eq:sta_delayed} \\
    \datavec(t) & = \oMat'\statevec(t), \label{eq:obs_delayed}
\end{align}
where $\lag \ge 0$ represents an explicit delay in the state update,
$\lMat' \in \mathbb{R}^{\sdim \times \sdim}$ governs the latent dynamics,
$\cMat' \in \mathbb{R}^{\sdim \times \cdim}$ captures the influence of the delayed input
$\ctrvec(t-\lag)$, and $\oMat' \in \mathbb{R}^{\datadim \times \sdim}$ maps the latent state
to the observation.
For simplicity, we do not include direct-feedthrough terms or additive noise.

The Markov parameters of this delayed system are defined by its impulse response to an input
that is zero everywhere except at a single time step. They are given by
\begin{align}
\nonumber
    h_j =
    \begin{cases}
        0, & \text{for } 1 \le j \le \lag, \\
        \oMat'(\lMat')^{j-\lag-1}\cMat', & \text{for } j > \lag.
    \end{cases}
\end{align}
The first $\lag$ zeros represent the delay, while the subsequent nonzero entries characterize the response after the delay.
It is well known that there exists a delay-free state-space realization
whose Markov parameters match this sequence.
In other words, the explicit delay can be included in an augmented state.

The next theorem shows that any minimal delay-free realization with matching Markov parameters
is input-output equivalent to the original delayed system.

\begin{theorem}
\label{thm:equivalence}
Suppose that the delayed system~\eqref{eq:sta_delayed}–\eqref{eq:obs_delayed} is minimal
(reachable and observable).
Let $(\lMat, \cMat, \oMat)$ be any minimal delay-free state-space realization whose Markov
parameters $\{\oMat \lMat^{j-1} \cMat\}_{j\ge 1}$ coincide with those of the delayed system.
Then the delayed and delay-free systems are input--output equivalent:
for any input sequence and any initial state of the delayed system,
there exists a corresponding initial state of the delay-free realization
such that the resulting output sequences coincide.
\end{theorem}

\mypara{Proof sketch of Theorem ~\ref{thm:equivalence}}
The claim follows from the standard
augmented-state construction. Define an augmented state that contains
the original latent state and a buffer of the past $\tau$ inputs. The
augmented transition matrix updates the original state using the oldest
buffered input, shifts the buffer by one step, and inserts the current
input into the buffer. This gives a delay-free state-space realization.
By construction, its first $\tau$ Markov parameters are zero and the
remaining Markov parameters coincide with those of the delayed system.
Therefore, for any input sequence, the two systems have the same
input--output behavior under corresponding initial states.
\qed

Theorem~\ref{thm:equivalence} simplifies the problem of identifying a delayed system to
estimating its possibly sparse sequence of Markov parameters from data.
This supports working in the standard state-space framework without explicitly
modeling delays, as long as the Markov parameters are correctly captured.
In particular, this supports the modeling choice in \method\ to perform inference forecasting entirely in terms of delay-free state-space models realized from
estimated Markov parameters.

\subsection{Moment Tensor Representation and Tensor Factorization}
\label{subsec:moment-tensor}

Next, we summarize the moment-based tensor representation used in the
online estimation algorithm of \method.
In the previous section, we introduced the system tensor $\systemtensor$
as a compact way to encode higher-order moments of the input-output data and
as the object for CP decomposition in \Mcollect\ and \Madapt.
Here we state a concise result from {\cite{LDSmeetsCTR}}, which shows that factorizing a properly constructed moment tensor is enough to recover
the Markov parameters of the underlying systems.

\begin{theorem}[Recovery of Markov parameters via moment tensor factorization~{\cite{LDSmeetsCTR}}]
\label{thm:moment-cp-compact}
Let $(\ctrvec(t),\datavec(t))_{t\ge 1}$ be generated by a (mixture of) linear
time-delay MIMO systems with regime-specific Markov parameters
$\{h^{(i)}_j\}_{j\ge 1}$.
Assume that the exogenous inputs are independent and identically distributed,
with a distribution that satisfies suitable moment conditions, and that the resulting moment tensor is well-defined.
Then there exists a polynomial-time procedure that constructs a tensor
$\systemtensor$ from empirical moments of $(\ctrvec(t),\datavec(t))$ such that,
under mild identifiability conditions, any sufficiently accurate CP decomposition
of $\systemtensor$ recovers the Markov parameters $\{h^{(i)}_j\}_{j\ge 1}$ of
the constituent systems, up to permutation and scaling of the components.
\end{theorem}

\begin{proof}
See \cite{LDSmeetsCTR} for the complete proof.
\end{proof}

This theorem gives the theoretical justification for the tensor factorization
step in \method. The CP factors obtained by applying ALS to the empirical
system tensor in \Madapt\ can be seen as estimates of the
regime-specific Markov parameters.
Combined with Theorem~\ref{thm:equivalence}, this ensures that the subsequent
realization step produces delay-free models whose input-output behavior
matches that of the underlying time-delay systems.
\begin{figure*}[t]
    \centering
    \includegraphics[width=\linewidth]{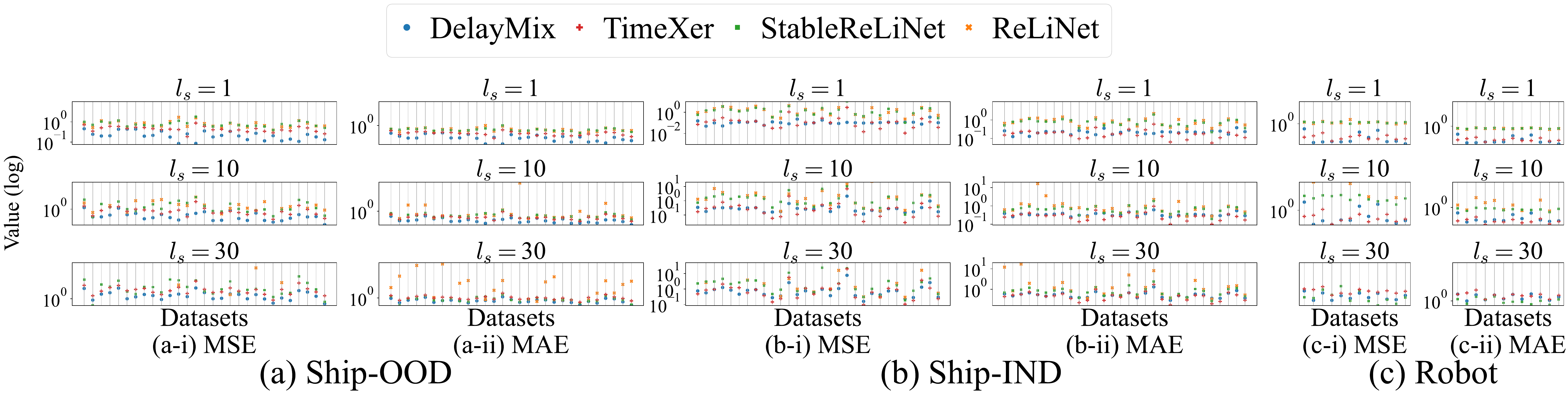}
    \caption{
        Forecasting performance comparison.
        On the delayed mixture benchmark, \method\ outperforms competing methods in terms of
        mean squared error (MSE) and mean absolute error (MAE).
    }
    \label{fig:result}
\end{figure*}
\subsection{Computational Complexity of \method}
\label{subsec:complexity}

Finally, we analyze the computational complexity of the proposed online algorithm
\method.
The dominant computational cost arises from the ALS-based tensor decomposition
used to update the system tensor representation, along with the following
realization of state-space models.

\begin{lemma}[Time complexity of \method]
\label{thm:time}
Given a new incoming tensor (i.e., an updated system tensor $\systemtensor$), the time complexity of one update of \method is
\begin{align}
\nonumber
        \mathcal{O}\bigl(
        i\,(8s^3\datadim^3\cdim^3\trank
         + 6s\datadim\cdim\,\trank^2)
        + \sdim^3\len_c
        + \sdim^2\cdim\lstep
        + \sdim^2\lstep
        + \sdim\lstep\datadim
    \bigr),
\end{align}
where $i$ is the number of ALS iterations for tensor decomposition, $s$ is the
maximum lag used to construct $\systemtensor$, $\trank$ is the CP rank, $\len_c$
is the length of the current data window, and $\lstep$ is the forecasting horizon.
\end{lemma}

\mypara{Proof sketch of Lemma \ref{thm:time}}
Let $p=2sdd_c$ be the mode size of the
constructed third-order system tensor. One ALS iteration for a rank-$R$
CP decomposition costs $O(p^3R+pR^2)$, and repeating it for $i$
iterations gives the first term. The remaining terms are the costs of
state inference on the current window and simulation over the forecasting
horizon.
\qed

The computational cost of \method mainly depends on the ALS-based tensor
decomposition, especially on the number of iterations $i$ needed for convergence.
To reduce this cost in practice, the full decomposition step is only triggered when
there is a significant mismatch between the current model and the newly updated system tensor
is detected. This helps avoid unnecessary updates.
Also, because the data changes gradually over time, consecutive system
tensors are similar. As a result, initializing ALS with the previous factor estimates usually
leads to fast convergence (small $i$).
These properties allow \method\ to work efficiently in streaming settings while
maintaining accurate modeling and forecasting performance.

%% file: 040experiments.tex
In this section, we describe the performance of \method using real datasets. The experiments were designed to answer the following questions about \method:
\begin{description}
    \item[\textit{Q1.Accuracy:}] How accurately does it predict future events?
    \item[\textit{Q2.Scalability:}] How does it scale in terms of computational time and memory?
\end{description}
\subsection{Experimental Setup}
\begin{figure}[t]
    \centering
    \includegraphics[width=\linewidth]{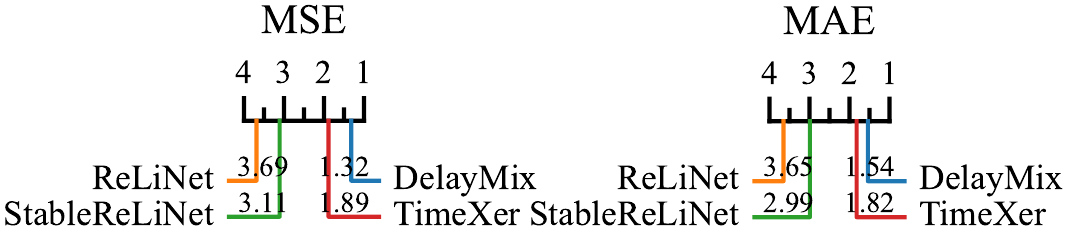}
    \caption{
        Critical difference diagram of datasets in terms of MSE and MAE.
    }
    \label{fig:cd}
\end{figure}
\mypara{Datasets} 
To evaluate \method, we use three real-world datasets. 

For the evaluation of predictive accuracy (Q1) and scalability (Q2), we used the following three publicly available real-world datasets:
\begin{itemize}
\item 
\textbf{Ship-OOD/Ship-IND} \cite{Shipdataset} is a 4-DOF ship maneuvering dataset under environmental disturbances. The 4 exogenous variables include propeller speed, rudder angles, and wind; outputs are also 4 (linear and angular velocities). We use both in-distribution (Ship-IND) and out-of-distribution (Ship-OOD) test sets.
\item \textbf{Robot} \cite{Robotdataset} is a dataset of a 6-DOF industrial robot arm, with 6 exogenous inputs (motor torques) and 6 outputs (joint angles).
\end{itemize}

\mypara{Baseline Methods}
We compare \method with state-of-the-art models that explicitly incorporate exogenous variables, including StableReLiNet, ReLiNet \cite{ReLiNet}, and TimeXer \cite{TimeXer}. Many existing forecasting models either omit exogenous inputs or treat them as endogenous, limiting their ability to model structured input-output relationships. Accordingly, we exclude such models from our main comparison.

\mypara{Hyperparameter Setting of \method}
We selected \(\rho \in \{0.5,0.6,0.7,0.8,0.9,1.0\}\) and
\(R \in \{2,4,8,10\}\) using the validation set, and fixed \(s=3\).
\subsection{Q1: Accuracy}
We evaluated \method using mean squared error (MSE) and mean absolute error (MAE), where lower values indicate better accuracy. Results are averaged over five trials for robustness. As shown in Fig.~\ref{fig:result}, \method consistently yields lower errors than baselines. The critical difference diagram in Fig.~\ref{fig:cd}, based on the Wilcoxon-Holm test~\cite{WilcoxonHolmmethod}, confirms these improvements are statistically significant. Table~\ref{table:win_count} reports how often each method ranks best, further highlighting the advantage of \method.

A key strength of \method is its ability to adaptively estimate a mixture of time-delay systems in streaming settings, enabling it to capture dynamic and heterogeneous temporal patterns without prior knowledge of delay structures or access to past data. While TimeXer performs well on the stationary Ship-IND dataset, it falls short on the non-stationary Ship-OOD data, where adaptive modeling is crucial. In contrast, models like StableReLiNet and ReLiNet tend to degrade under streaming updates, likely due to unstable parameter adaptation. For example, ReLiNet shows noticeably lower accuracy in several scenarios (see Fig.~\ref{fig:result}).
These results show that \method enables accurate real-time prediction in dynamic streaming settings.

\input{TABLE/table_win_count}

\subsection{Q2: Scalability}
Next, we evaluated the performance of \method in terms of computation time. The left column of Fig. \ref{fig: scale_t} shows the wall clock time of the experiments performed on each dataset. Thanks to an effective parameter estimation algorithm, the computation time of \method is independent of the length of the data stream. Several small spikes are due to the execution of \Madapt. CP decomposition for updating the regime database increases the computation time for that time step. The right column of Fig. \ref{fig: scale_t} shows the average computation time for the entire tensor stream. Notably, TimeXer and ReLiNet were run in a relatively fast environment, as all computations, including forecasting future values and parameter updating using lightweight single-sample gradients, are performed on the GPU. However, \method is still overwhelmingly faster than them. This demonstrates that \method achieves a high trade-off between accuracy and efficiency.

\begin{figure}[t]
    \centering
    \includegraphics[width=\linewidth]{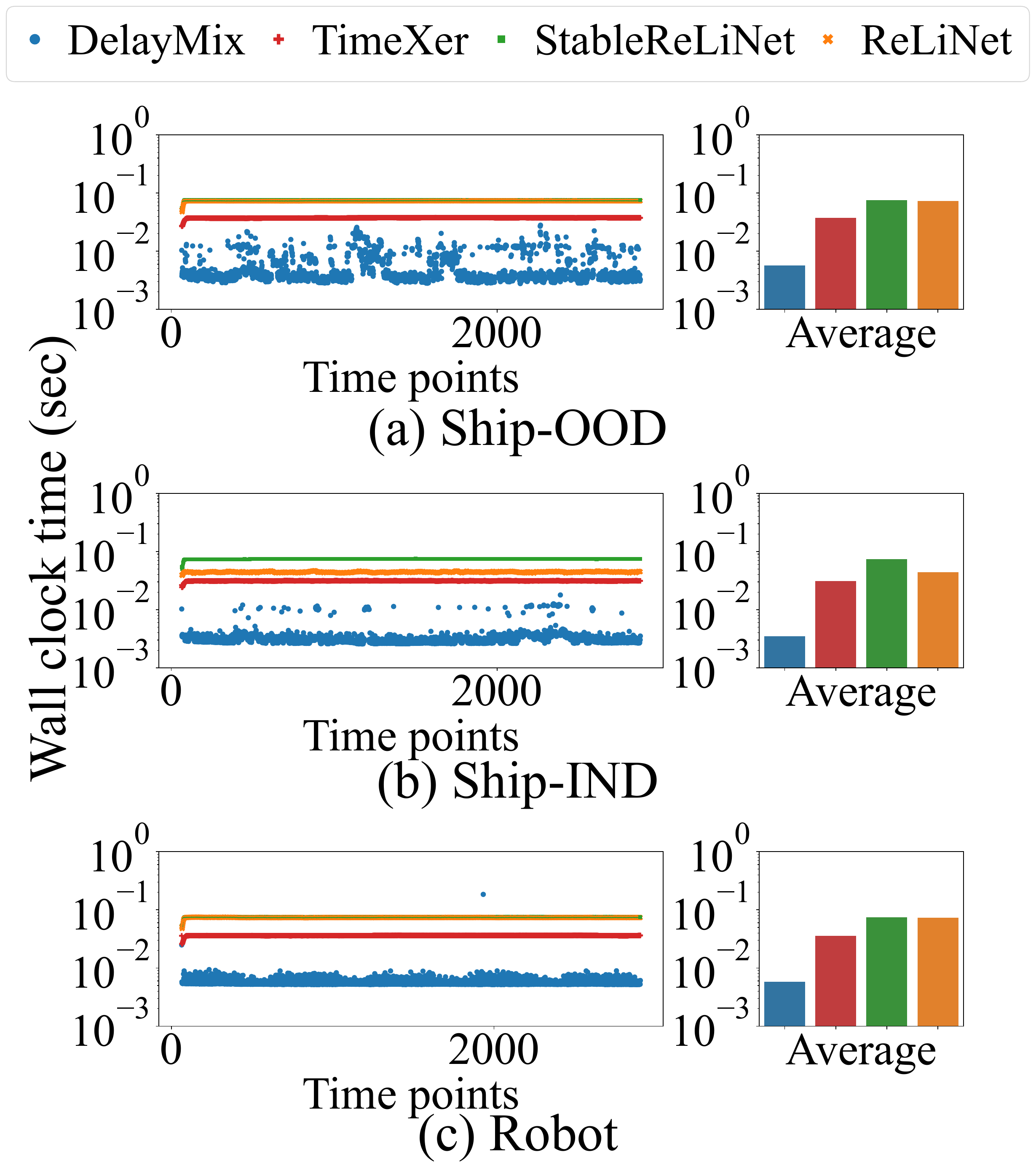}
    \caption{
        Efficiency of streaming \method: Our method is consistently faster than other baseline methods.
    }
    \label{fig: scale_t}
\end{figure}

%% file: TABLE/table_win_count.tex
\begin{table}[t]
\centering
\resizebox{1.0\linewidth}{!}{
\begin{tabular}{ll|cccccccc}
\toprule
\multicolumn{2}{c|}{\textbf{Method}} & \multicolumn{2}{c}{\textbf{\method}} & \multicolumn{2}{c}{\textbf{TimeXer}} & \multicolumn{2}{c}{\textbf{ReLiNet}} & \multicolumn{2}{c}{\textbf{StableReLiNet}} \\
\cmidrule(lr){1-2} \cmidrule(lr){3-4} \cmidrule(lr){5-6} \cmidrule(lr){7-8} \cmidrule(lr){9-10}
\textbf{Dataset} & $\boldsymbol{l_s}$ & \textbf{MSE} & \textbf{MAE} & \textbf{MSE} & \textbf{MAE} & \textbf{MSE} & \textbf{MAE} & \textbf{MSE} & \textbf{MAE} \\

\midrule
\multirow{3}{*}{Ship-OOD} 
& 1  & \textbf{28} & \textbf{26} & \underline{1} & \underline{3} & 0 & 0 & 0 & 0 \\
& 10 & \textbf{26} & \textbf{28} & \underline{2} & \underline{1} & 0 & 0 & 1 & 0 \\
& 30 & \textbf{27} & \textbf{25} & 0 & 0 & \underline{1} & 1 & \underline{1} & \underline{3} \\
\midrule
\multirow{3}{*}{Ship-IND} 
& 1  & \textbf{18} & \underline{9}  & \underline{12} & \textbf{21} & 0 & 0 & 0 & 0 \\
& 10 & \textbf{16} & \underline{6}  & \textbf{14} & \textbf{24} & 0 & 0 & 0 & 0 \\
& 30 & \textbf{16} & \underline{8}  & \underline{13} & \textbf{22} & 1 & 0 & 0 & 0 \\
\midrule
\multirow{3}{*}{Robot} 
& 1  & \textbf{8} & \textbf{7} & 4 & 5 & 0 & 0 & 0 & 0 \\
& 10 & \textbf{7} & \underline{3}  & \underline{5} & \textbf{9} & 0 & 0 & 0 & 0 \\
& 30 & \underline{5} & \underline{1}  & 0 & 0 & 0 & 0 & \textbf{7} & \textbf{11} \\
\bottomrule
\end{tabular}
}
\caption{$1^{\text{st}}$ Count in dataset (higher is better).}
\label{table:win_count}
\end{table}

%% file: 080related.tex
This section reviews related work on mixtures of dynamical systems and time-series forecasting.

\mypara{Learning Mixtures of Dynamics}
Traditionally, dynamical system estimation has centered on a single system \cite{SINDy,ode1,si_icml_18,KDD2023_PDE,mio-sindy,LaNoLem,ICLR2025_si,ICMR2025_si}, but recent advances enable learning mixed systems—where multiple dynamical regimes coexist—from unlabeled trajectory data.
For example, \cite{MixtureofLDS} proposes a spectral algorithm that learns a mixture from short unlabeled trajectories with end-to-end guarantees, and \cite{LDSmeetsCTR} introduces a tensor-based framework that recovers LDS mixtures via Markov parameters estimated from high-order moments.
These approaches are theoretically appealing but inherently offline: constructing the required moment tensors assumes access to the full dataset, which is impractical in streaming scenarios such as industrial process control.
Moreover, they handle delays only implicitly and do not explicitly model unknown or time-varying delays, making it difficult to adapt to abrupt changes in system dynamics.
\method builds on the framework of \cite{LDSmeetsCTR} and extends it to streaming data by incrementally updating a system tensor and adaptively estimating multiple time-delay systems from incoming observations.

\mypara{Time-Series Forecasting Methods}
Classical forecasting methods such as ARIMA and Kalman filters~\cite{ARIMA,LDS,Dynammo} remain standard tools, and recent deep learning models~\cite{Informer,DeepAR,DLinear,PatchTST,Fredformer} have achieved strong empirical performance.
However, many of these methods either ignore exogenous variables or treat them in the same way as endogenous variables, although exogenous inputs often play a distinct and crucial role in prediction.
To better exploit exogenous information, ARIMAX~\cite{ARIMAX} extends ARIMA with exogenous regressors, while recent neural approaches such as ReLiNet~\cite{ReLiNet} and TimeXer~\cite{TimeXer} incorporate exogenous signals through switched linear systems and Transformer-based attention, respectively, achieving competitive results.
However, these models generally do not model time delays explicitly and tend to struggle in streaming settings due to high computational cost and frequent retraining, whereas \method is designed for online operation: it maintains a compact summary of input–output relationships, explicitly captures delay dynamics, and updates models efficiently in real time.

%% file: 090conclusion.tex
We introduce \method, an online framework for modeling dynamic mixtures of time-delay systems from streaming time series. Our approach addresses key challenges in streaming scenarios, including nonstationarity, unknown delays, and limited resources, while maintaining both accuracy and scalability over long horizons. \method incrementally builds and decomposes high-order moment tensors to extract regime-specific Markov parameter sequences, which efficiently capture system dynamics and input-output delays in a compact, delay-sensitive representation. This allows for delay-aware identification even when regimes switch, without assuming known change points or fixed delay structures. The framework does not require direct access to the system or prior knowledge of delay lengths. Instead, it uses moment-based summaries that are updated online, so it does not need to store the entire history, keeping memory use low and computation efficient. \method can also update its set of candidate systems when the current model no longer fits new observations, helping it adapt to changing dynamics while avoiding unbounded memory growth or switching costs. Our experiments with real-world datasets show that \method consistently achieves better prediction accuracy and computational efficiency than other methods by accurately tracking changing delays and adapting to regime shifts. 

%% file: 0A_appendix.tex
\section{Symbols and Definitions}
\input{TABLE/table_symbol}
Table \ref{table: define} lists the symbols and definitions used in this paper.
\section{Details of algorithms}
\label{apseq: ho-kalman}
\subsection{Ho-Kalman algorithm}
Here, we show the overrview of Ho-Kalman algorithm \cite{HoKalman}.
\input{ALG/alg_hokalman}
\subsection{Algorithm Overview}
Here, we show the overrview of \method.
\input{ALG/alg_overview}
\subsection{Generate: Forecasting}
\label{apseq:kalman}
Here, we describe the following the forward-pass (filter) and backward-pass (smoother) equations.
\paragraph{Forward Pass (Kalman Filter):}
\begin{align}
    \label{eq: fw_st}
    \fwmu(t) &= \lMat\filmu(t-1)+\cMat\ctrvec(t-1), \\
    \hat{\msCov}(t) &= \lMat\msCov(t-1)\lMat^\top+\sCov, \\
    \mathbf{K}(t) &= \hat{\msCov}(t)\oMat^\top(\oMat\hat{\msCov}(t)\oMat^\top+\oCov)^{-1},\\
    \filmu(t) &= \fwmu(t)+\mathbf{K}(t)(\datavec(t)-\oMat\fwmu(t)),\\
    \label{eq: fw_ed}
    \msCov(t) &= (\mathbf{I}-\mathbf{K}(t)\oMat)\hat{\msCov}(t).
\end{align}

\paragraph{Backward Pass (RTS Smoother):}
\begin{align}
    \label{eq: bw_st}
    \msB(t) &= \msCov(t)\lMat^\top\hat{\msCov}(t+1)^{-1}, \\
    \label{eq: bw_ed}
    \hat{\statevec}(t) &= \filmu(t)+\msB(t)(\hat{\statevec}(t+1) - \fwmu(t+1)).
\end{align}

\section{Proofs}
\label{apseq:proofs}
\subsection{Proof of Theorem \ref{thm:equivalence}}
\label{apseq:equivalence-proof}

\begin{lemma}[Delay-free realization from Markov parameters]
\label{thm:markov}
Let the Markov parameters of the delayed system~\eqref{eq:sta_delayed}–\eqref{eq:obs_delayed}
be given by $\{g_j\}_{j\ge 1}$.
Then there exists a delay-free state-space model
\begin{align}
    \tilde{\statevec}(t+1) & = \lMat \tilde{\statevec}(t) + \cMat \ctrvec(t), \\
    \estvec(t) & = \oMat \tilde{\statevec}(t),
\end{align}
with possibly higher-dimensional latent state $\tilde{\statevec}(t)$ and suitable matrices
$\lMat, \cMat, \oMat$ such that
\begin{align}
              \nonumber
    g_j = \oMat \lMat^{j-1} \cMat \quad \text{for all } j \ge 1.
\end{align}
In other words, the delay-free model reproduces exactly the same Markov parameters as
the original delayed system.
\end{lemma}

\begin{proof}
Because the original system has an explicit input delay of $\lag$ steps, the input $\ctrvec(t)$
only affects the output from time $t+\lag$ onward.
Its input–output behavior is therefore fully characterized by the Markov parameter sequence
$\{g_j\}_{j\ge 1}$, which defines the impulse response
\begin{align}
              \nonumber
    \estvec(t) = \sum_{j=\lag}^{\infty} g_j \,\ctrvec(t - j).
\end{align}

We construct an equivalent delay-free state-space model by embedding the delayed inputs into
an extended state vector.
Specifically, define the extended state to stack the last $\lag$ inputs:
\begin{align}
              \nonumber
    \tilde{\statevec}(t)
    :=
    \begin{bmatrix}
        \ctrvec(t - \lag + 1)^\top &
        \ctrvec(t - \lag + 2)^\top &
        \cdots &
        \ctrvec(t)^\top
    \end{bmatrix}^\top.
\end{align}
With this construction, the delay is absorbed into the state update, and the output $\estvec(t)$
can be written as a linear function of $\tilde{\statevec}(t)$ without explicit delay.

As a result, the convolution representation induced by $\{g_j\}$ can be realized by a standard
delay-free state-space model.
Classical realization techniques such as the Ho–Kalman algorithm or subspace methods applied
to the Markov parameter sequence $\{g_j\}$ yield matrices $(\lMat, \cMat, \oMat)$ that satisfy
\begin{align}
              \nonumber
    g_j = \oMat \lMat^{j-1} \cMat \quad \text{for all } j \ge 1.
\end{align}
Hence the delayed system and the constructed delay-free system share identical Markov parameters,
which proves the claim.
\end{proof}

\begin{proof}[Proof of Theorem \ref{thm:equivalence}]
By Lemma~\ref{thm:markov}, the delayed system~\eqref{eq:sta_delayed}–\eqref{eq:obs_delayed}
admits a delay-free state-space realization whose Markov parameters $\{g_j\}_{j\ge 1}$ coincide
with those of the delayed system.
The input–output map of a linear time-invariant system is completely determined by its Markov
parameters (equivalently, by its transfer function).
Therefore any two minimal state-space realizations with the same Markov parameters are related
by a similarity transformation and produce identical outputs for any input sequence.
Since both the delayed system and $(\lMat,\cMat,\oMat)$ are minimal and share the same
Markov parameters, they are input–output equivalent.
\end{proof}

\subsection{Proof of Lemma \ref{thm:time}}
The most time-consuming part in \Mcollect and \Madapt is the computation of tensor decomposition, and its complexity is $\ord(i(8s^3\datadim^3\cdim^3\trank+6s\datadim\cdim \trank^2))$. The calculation of \Fpred is $\ord(\sdim^3\len_c+\sdim^2\cdim\lstep+\sdim^2\lstep+\sdim\lstep\datadim)$. Thus the total computation is $\ord(i(8s^3\datadim^3\cdim^3\trank+6s\datadim\cdim \trank^2) + \sdim^3\len_c+\sdim^2\cdim\lstep+\sdim^2\lstep+\sdim\lstep\datadim)$.

\section{Limitations and Future Work}
\input{070limitations}

\section{Additional Results}
\label{apseq:addtionalresults}
\begin{figure}[t]
    \centering
    \includegraphics[width=0.95\linewidth]{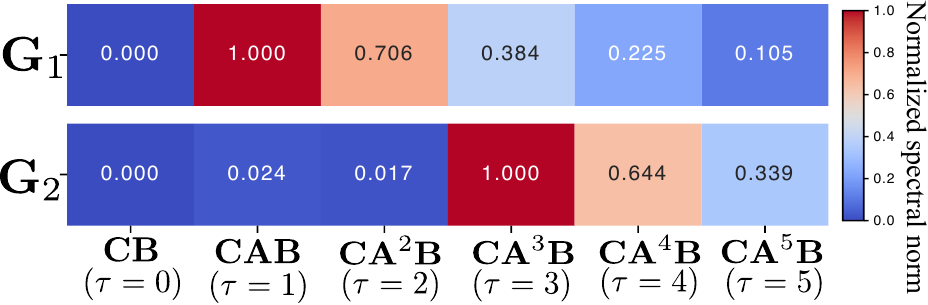}
    \caption{
          Effectiveness of the results: Normalized spectral norms of Markov parameters estimated by \method for synthetic systems with input delays of 1 and 3 ($\lag=1,3$). Each row corresponds to a Markov parameter (i.e., a system response at a time step), and each column within a block matrix represents a component of the parameter. The magnitude of each entry reflects the influence of the input at that position. Delays are revealed as leading low-magnitude entries, allowing \method to estimate the underlying system delays without prior knowledge.
    }
    \label{fig:intro}
\end{figure}
\input{TABLE/table_numerical}

This section provides additional empirical results that complement the experiments in Section~\ref{sec: 04_exp}. We first report aggregated forecasting performance across all benchmarks, then illustrate how \method recovers input–delay structures and tracks regime changes on synthetic data, and finally examine the sensitivity of \method to key hyperparameters.
\subsection{Numerical results}
\label{apseq: numerical}
Table~\ref{table:error_results_highlighted} reports the mean and standard deviation of forecasting errors across all datasets for \method and its baselines, with the best and second-best results shown in \textbf{bold} and \underline{underline}, respectively. As discussed in Section~\ref{sec: 04_exp}, our method outperforms state-of-the-art approaches.
\subsection{Effectiveness of \method}
Fig.\ref{fig:intro} shows the normalized spectral norms of the Markov parameters estimated by \method on synthetic datasets generated from two systems with input delays $\lag=1$ and $\lag=3$. Each row corresponds to a time step in the system’s response, and each column in the block matrix represents a Markov parameter component. The spectral norm of each column reflects the strength of the input’s influence on that component.

A system delay appears as a sequence of zeros or low-magnitude values at the beginning (left side) of the matrix. As shown, \method accurately identifies delays of 1 and 3 by detecting these characteristic patterns without any prior knowledge. This illustrates its ability to recover input delay structures and adaptively model a mixture of time-delay systems through estimated Markov parameters. While the delay patterns are clearly visible in this experiment, their clarity may be reduced under high noise, overlapping delays, or frequent switching. Nevertheless, we observe that they remain identifiable in practice.
\subsection{Sensitivity to Hyperparameters}
We evaluated the sensitivity of \method to two key hyperparameters ($R$ and $\rho$)on the Ship-OOD dataset with forecasting horizon fixed to $l_s = 30$.
Figure~\ref{apfig:hp_sensitivity} reports the mean MAE and MSE when varying $R \in \{2,4,8,10\}$ and $\rho \in \{0.5,0.6,0.7,0.8,0.9,1.0\}$.
Both metrics change only mildly across the explored ranges, and no setting leads to a substantial degradation in performance, indicating that \method is robust to reasonable choices of $R$ and $\rho$.
We observe slightly better errors for intermediate values, but the overall trends remain flat, suggesting that \method does not require fine-tuning of these hyperparameters to achieve strong forecasting performance on Ship-OOD.
\begin{figure}
    \centering
    \includegraphics[width=\linewidth]{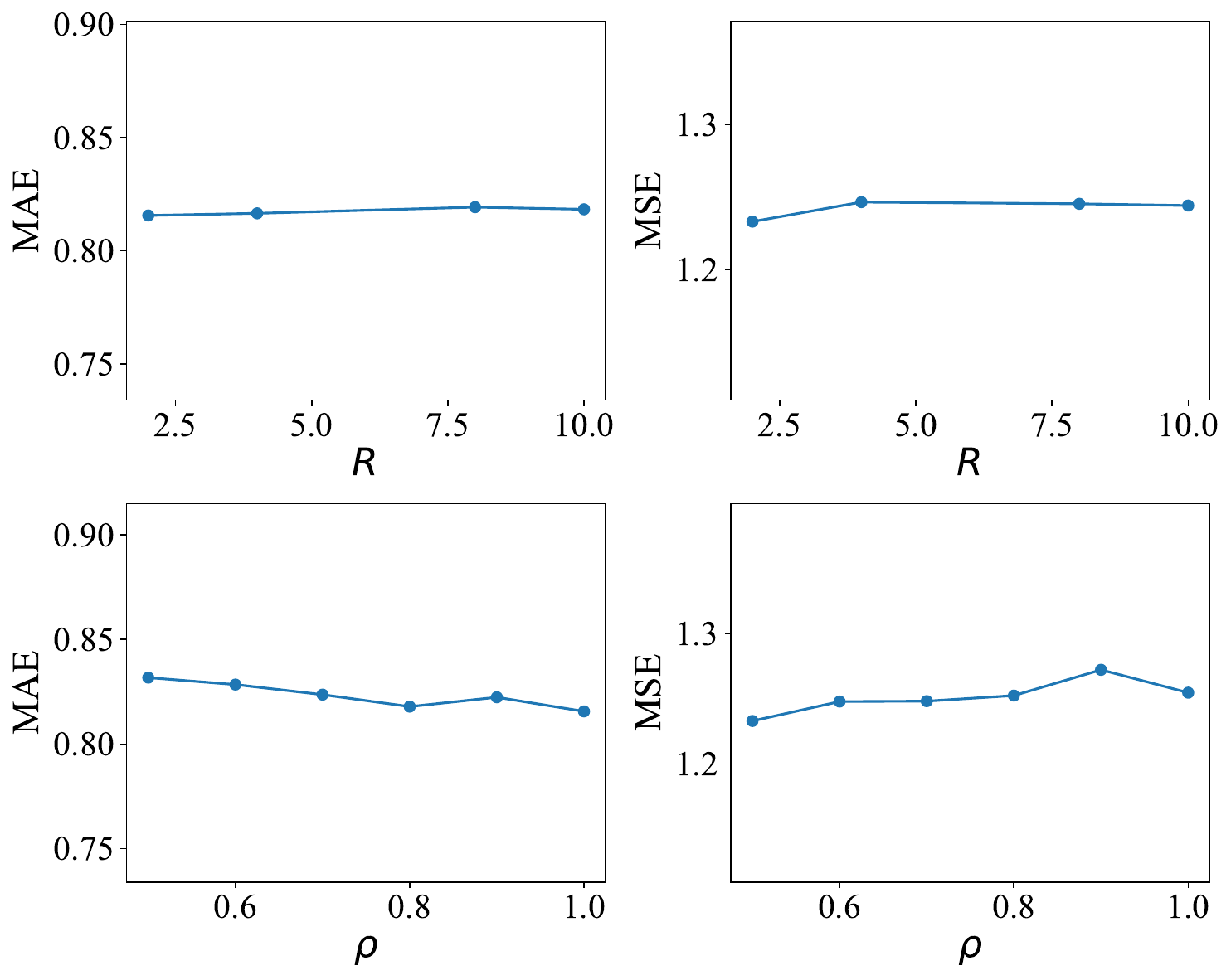}
    \caption{Hyperparameter studies.}
    \label{apfig:hp_sensitivity}
\end{figure}

\section{Experimental Setting}
\label{apseq: setting}
\mypara{Computing Infrastructure}
The configuration includes 2 * Xeon Gold 6258R 2.7Hz CPU, 12 * 64GB DDR4 RAM (768GB), and NVIDIA RTX A6000 48GB GPU, which is sufficient for all the baselines.

\mypara{Metrics}
All models are typically assessed using cumulative mean-squared errors (MSE) and mean-absolute errors (MAE), which means the model's performance is judged based on its total errors accumulated throughout the entire learning sequence.

\mypara{Hyperparameter Setting of \method}
We searched for the best results in $\rho \in \{0.5, 0.6, 0.7, 0.8, 0.9, 1.0\}$ and $\trank \in \{2,4,8,10\}$ to minimize the errors in the validation dataset and fixed $s$ to $3$. 

\mypara{Baselines}
The baseline learning rate was selected from the range $[1e-3, 3e-3, 1e-2, 3e-2]$ to minimize the errors in the validation data. The other parameters are the same as those that we used the default parameters provided in the authors' codes.

%% file: TABLE/table_symbol.tex
\begin{table}[H]
\centering
\caption{Symbols and definitions.}
\label{table: define}
\resizebox{1.0\linewidth}{!}{
\begin{tabular}{l|l}
\toprule
Symbol & Definition \\
\midrule
${\datadim}$
    & number of dimensions of data\\
${\cdim}$
    & number of dimensions of exogenous variables\\
${\sdim}$
    & number of state dimensions\\
${\trank}$
    & number of time-delay systems\\
${\newt}$
    & current time point\\
${l_c}$
    & current window size, i.e., $\newt-t=l_c$\\
${\lstep}$
    & forecasting step size\\
\midrule
${\datac}$
    & current data, i.e., $\datac = \data[t:\newt]$\\
${\cctr}$
    & current exogenous variab, i.e., $\cctr = \ctr[t:\newt+\lstep]$\\
\midrule
$\statevec(t)$
    & latent states at time point $t$\\  
$\estvec(t)$
    & estimated values at time point $t$\\
$\mathbf{G}$
    & Markov parameters \\
\midrule
$\systemtensor$
    & a third-order moment statistics tensor, $\systemtensor\in \mathbb{R}^{2s\datadim\cdim \times 2s\datadim\cdim \times 2s\datadim\cdim}$ \\
${\lMat}$
    & linear transition matrix, ${\sdim}\times{\sdim}$\\
${\cMat}$
    & projection matrix of exogenous inputs, ${\sdim}\times{\cdim}$\\
${\oMat}$
    & observation matrix, ${\datadim}\times{\sdim}$\\

\bottomrule
\end{tabular}
}
\end{table}

%% file: ALG/alg_hokalman.tex
\begin{algorithm}[H]
    \caption{Ho-kalman ($\mathbf{G}$)}
    \label{alg:Ho-Kalman}
    \begin{algorithmic}[1]
    \STATE {\bf Input:} (a) A markov parameter $\mathbf{G}=\{\mathbf{g}_1,...,\mathbf{g}_{2s}\}$ \\
    \STATE {\bf Output:} System parameters $\{\lMat, \cMat, \oMat\}$\\
    \STATE Form the Hankel matrix $\hat{H} \in \mathbb{R}^{ms \times p(s+1)}$ from $\mathbf{G}$ as
    \[
    \hat{H} = 
    \begin{bmatrix}
    \mathbf{g}_1 & \mathbf{g}_2 & \cdots & \mathbf{g}_{s+1} \\
    \mathbf{g}_2 & \mathbf{g}_3 & \cdots & \mathbf{g}_{s+2} \\
    \vdots & \vdots & \ddots & \vdots \\
    \mathbf{g}_s & \mathbf{g}_{s+1} & \cdots & \mathbf{g}_{2s}
    \end{bmatrix}
    \]
    \STATE $\hat{H}^- \in \mathbb{R}^{ms \times ps} \leftarrow$ first $ps$ columns of $\hat{H}$
    \STATE $\hat{L} \in \mathbb{R}^{ms \times ps} \leftarrow$ rank $n$ approximation of $\hat{H}^-$ via SVD
    \STATE $U, \Sigma, V = \mathrm{SVD}(\hat{L})$
    \STATE $\hat{O} \in \mathbb{R}^{ms \times n} \leftarrow U \Sigma^{1/2}$
    \STATE $\hat{Q} \in \mathbb{R}^{n \times ps} \leftarrow \Sigma^{1/2} V^\top$
    \STATE $\hat{C} \leftarrow$ first $m$ rows of $\hat{O}$
    \STATE $\hat{B} \leftarrow$ first $p$ columns of $\hat{Q}$
    \STATE $\hat{H}^+ \in \mathbb{R}^{ms \times ps} \leftarrow$ last $ps$ columns of $\hat{H}$
    \STATE $\hat{A} \leftarrow \hat{O}^\dagger \hat{H}^+ \hat{Q}^\dagger$
    \RETURN $\hat{A} , \hat{B}, \hat{C}$
    \end{algorithmic}
    \normalsize
\end{algorithm}

%% file: ALG/alg_overview.tex
\begin{algorithm}[H]
   \small
    \caption{\method ($\datac, \cctr, \systemtensor_p, \regimeset_p$)}
    \label{alg:overview}
    \begin{algorithmic}[1]
    \STATE {\bf Input:} (a) Current data window $\datac=\{\datavec(t)\}$ \\
    \ \ \ \ \ \ \ \ \ \ \ \
    (b) Current exogenous window $\cctr=\{\ctrvec(t)\}$ \\
    \ \ \ \ \ \ \ \ \ \ \ \
    (c) Previous system tensor $\systemtensor_p$ \\
    \ \ \ \ \ \ \ \ \ \ \ \
    (d) Previous system parameters $\regimeset_p$
    \STATE {\bf Output:} (a) Current system parameters $\regimeset$ \\
     \ \ \ \ \ \ \ \ \ \ \ \ \ \ \
    (b) Updated system tensor $\systemtensor$ \\
    \ \ \ \ \ \ \ \ \ \ \ \ \ \ \
    (c) Future values $\fdata$

    \STATE /* \Mcollect: update global system tensor */
    \STATE $\systemtensor \gets \systemtensor_p$
    \STATE Set $k_{\max}$, and let $p = \datadim\cdim$, $D=k_{\max}p$
    \STATE Let $\mathcal{J}(k)$ denote the index range $\{(k-1)p+1,\dots,kp\}$

    \FORALL{$(k_1,k_2,k_3)\in\{1,\dots,k_{\max}\}^3$}
        \FORALL{admissible sub-window starts $\tau$ in the current window}
            \STATE $t_1 \gets \tau+k_1$
            \STATE $t_2 \gets \tau+k_1+k_2+1$
            \STATE $t_3 \gets \tau+k_1+k_2+k_3+2$
            \STATE $\tilde t_1 \gets \tau$, \ $\tilde t_2 \gets \tau+k_1+1$, \ $\tilde t_3 \gets \tau+k_1+k_2+2$
            \STATE $\mone \gets \mathrm{vec}\!\left(\datavec(t_1)\ctrvec(\tilde t_1)^\top\right)$
            \STATE $\mtwo \gets \mathrm{vec}\!\left(\datavec(t_2)\ctrvec(\tilde t_2)^\top\right)$
            \STATE $\mthree \gets \mathrm{vec}\!\left(\datavec(t_3)\ctrvec(\tilde t_3)^\top\right)$
            \STATE $\systemtensor[\mathcal{J}(k_1),\mathcal{J}(k_2),\mathcal{J}(k_3)] \gets
                   \systemtensor[\mathcal{J}(k_1),\mathcal{J}(k_2),\mathcal{J}(k_3)] + \mthree\otimes\mtwo\otimes\mone$
        \ENDFOR
    \ENDFOR

    \STATE /* Compute the error between the current system and the current data */
    \STATE Compute predictions $\{\estvec(j)\}$ from $\regimeset_p$ on the current window (e.g., via Kalman filtering; see Appendix~\ref{apseq:kalman}).
    \STATE $L_c \gets \sum^{t+t_n}_{j=t} \|\datavec(j)-\estvec(j)\|_2$

    \STATE /* The system $\regimeset_p$ is not updated if it sufficiently fits the current data. */
    \IF{$L_c < \rho$}
        \STATE $\regimeset \gets \regimeset_p$
    \ELSE
        \STATE /* \Madapt */
        \STATE $\{\fone_i, \ftwo_i, \fthree_i\}_{i=1}^{\trank} \gets$ CP-Decomposition($\systemtensor$)
        \FOR {$i=1$ to $\trank$}
            \STATE /* Construct Markov parameters $\mathbf{G}_i$ \cite{LDSmeetsCTR} */
            \STATE $\mathcal{T} \gets \fone_i\otimes\ftwo_i\otimes \fthree_i$
            \STATE Compute the Frobenius norm of $\mathcal{T}$ along modes 2--3 and obtain $\mathbf{v}_i$
            \STATE $\mathbf{G}_i \gets$ Reshape$\!\left(\mathbf{v}_i/\|\mathbf{v}_i\|^{2/3}\right)$
            \STATE $\iparams \gets$ Ho-Kalman($\mathbf{G}_i$)
        \ENDFOR
        \STATE /* Identify the system that best explains $\datac$ */
        \STATE $\regimeset \gets \underset{\iparams}{\argmin}\ \sum^{t+t_n}_{j=t} \|\datavec(j)-\estvec(j)\|_2$
    \ENDIF

    \STATE /* \Fpred */
    \STATE Compute $\hat{\statevec}(0)$ using $\regimeset$ (see Appendix~\ref{apseq:kalman}).
    \STATE Generate future values $\fdata$ from $\hat{\statevec}(0)$ and future inputs.
    \RETURN $\regimeset, \systemtensor, \fdata$
    \end{algorithmic}
    \normalsize
\end{algorithm}

%% file: 070limitations.tex
One limitation of our approach is its reliance on the tensor decomposition via ALS (Alternating Least Squares), which is known to be sensitive to initialization and may converge to local optima. However, in our streaming setting, the decomposition is only triggered when a significant mismatch between the current system and the new data is detected. Since the data typically evolves smoothly, we initialize ALS using the previous decomposition result, which significantly reduces the number of iterations and helps maintain stability in practice. Empirically, we observe that this warm-start strategy leads to consistent convergence with minimal sensitivity.

Another limitation is that our framework currently targets a mixture of linear time-delay systems. Extending it to nonlinear or hybrid systems is a promising direction for future work. One possible approach is to integrate kernelized Markov parameters or to replace linear system approximations with local surrogate models, which can better capture nonlinearity in physical processes.

Finally, although our current experiments focus on low- to medium-dimensional systems, scalability to high-dimensional settings (e.g., robotics with vision input or complex multi-agent control) remains an open challenge. Investigating low-rank tensor compression or structure-aware decomposition techniques may help address this in future extensions.

%% file: TABLE/table_numerical.tex
\begin{table*}[t]
\centering
\caption{Prediction error (mean $\pm$ std) for each method on different datasets and lookback steps $l_s$. \textbf{Bold} is best, \underline{Underline} is second best.}
\label{table:error_results_highlighted}
\resizebox{\linewidth}{!}{
\begin{tabular}{ll|cccc|cccc}
\toprule
\multicolumn{2}{c|}{\textbf{Dataset}} & \multicolumn{4}{c}{\textbf{MSE}} & \multicolumn{4}{c}{\textbf{MAE}} \\
\cmidrule(lr){1-2} \cmidrule(lr){3-6} \cmidrule(lr){7-10}
\textbf{Name} & $\boldsymbol{l_s}$ & DelayMix & TimeXer & ReLiNet & StableReLiNet & DelayMix & TimeXer & ReLiNet & StableReLiNet \\
\midrule
Ship-OOD & 1  & \textbf{0.2304 ± 0.1088} & \underline{0.4504 ± 0.1787} & 0.7965 ± 0.3896 & 0.7997 ± 0.317 & \textbf{0.3183 ± 0.09494} & \underline{0.4657 ± 0.1025} & 0.677 ± 0.1583 & 0.6715 ± 0.1199 \\
         & 10 & \textbf{0.6387 ± 0.1782} & \underline{0.889 ± 0.2697} & 4.813e+05 ± 5.735e+06 & 1.2 ± 0.4905 & \textbf{0.5666 ± 0.07572} & \underline{0.6882 ± 0.1069} & 19.84 ± 223.3 & 0.7561 ± 0.1219 \\
         & 30 & \textbf{1.254 ± 0.33} & \underline{1.745 ± 0.4817} & 4.07e+29 ± 4.884e+30 & 2.097 ± 1.425 & \textbf{0.8313 ± 0.1051} & \underline{1.002 ± 0.1327} & 3.931e+12 ± 4.717e+13 & 0.9169 ± 0.2002 \\
\midrule
Ship-IND & 1  & \textbf{0.1151 ± 0.05543} & \underline{0.2379 ± 0.5534} & 1.91 ± 2.912 & 1.708 ± 2.384 & \underline{0.2197 ± 0.05467} & \textbf{0.1884 ± 0.1081} & 0.8456 ± 0.5367 & 0.7767 ± 0.4528 \\
         & 10 & \textbf{0.3491 ± 0.3649} & \underline{0.5512 ± 1.323} & 8.847e+04 ± 1.063e+06 & 1.812 ± 2.473 & \underline{0.3423 ± 0.07987} & \textbf{0.2911 ± 0.1666} & 2.648 ± 16.24 & 0.6635 ± 0.4104 \\
         & 30 & \textbf{0.9965 ± 1.369} & \underline{1.657 ± 4.177} & 7.6e+24 ± 9.246e+25 & 3.567 ± 6.298 & \underline{0.5465 ± 0.1937} & \textbf{0.5097 ± 0.2965} & 3.907e+10 ± 4.754e+11 & 0.8255 ± 0.4678 \\
\midrule
Robot    & 1  & \textbf{0.2473 ± 0.1564} & \underline{0.2486 ± 0.137} & 1.148 ± 0.349 & 1.151 ± 0.1363 & \textbf{0.2831 ± 0.09666} & \underline{0.2857 ± 0.08596} & 0.8065 ± 0.06142 & 0.8152 ± 0.05008 \\
         & 10 & 0.8125 ± 0.2698 & \textbf{0.7588 ± 0.2776} & 2.62e+05 ± 2.021e+06 & \underline{1.581 ± 0.2555} & \underline{0.624 ± 0.09015} & \textbf{0.561 ± 0.1188} & 13.56 ± 93.35 & 0.9188 ± 0.06162 \\
         & 30 & \underline{2.088 ± 0.2065} & 2.442 ± 0.368 & 4.99e+27 ± 3.865e+28 & \textbf{2.07 ± 0.8314} & \textbf{1.065 ± 0.04473} & \underline{1.096 ± 0.09403} & 8.525e+11 ± 6.603e+12 & 0.9991 ± 0.08178 \\
\bottomrule
\end{tabular}
}
\end{table*}